\documentclass{article}
\usepackage[accepted]{icml2019}

\usepackage{microtype}
\usepackage{graphicx}
\usepackage{booktabs} 
\usepackage{hyperref}       
\usepackage{amsfonts}       
\usepackage{nicefrac}       
\usepackage{blkarray}
\usepackage{float}
\usepackage{tikz}
\usetikzlibrary{calc}
\usepackage{amsthm}
\usepackage{dsfont}
\usepackage{mathrsfs}
\usepackage{amsmath}
\usepackage{autobreak}
\newtheorem{theorem}{Theorem}

\usepackage{mathtools}
\usepackage[algo2e,vlined]{algorithm2e}

\usepackage{graphicx}
\graphicspath{ {figures/} }
\usepackage{subfig}

\newcommand{\real}{\mathbb{R}}

\newcommand{\bx}{\mathbf{x}}

\newcommand{\bw}{\mathbf{w}}

\newcommand{\bM}{\mathbf{M}}

\usepackage{enumitem}

\begin{document}
	
	\twocolumn[
	\icmltitle{Learning Algorithms via Neural Logic Networks}

	\begin{icmlauthorlist}
		\icmlauthor{Ali Payani}{gt}
		\icmlauthor{Faramarz Fekri}{gt}
	\end{icmlauthorlist}
	
	\icmlaffiliation{gt}{Department of Electrical and Computer Engineering
		Georgia Institute of Technology}
	
	\icmlcorrespondingauthor{Ali Payani}{payani@ece.gatech.edu}
	\icmlcorrespondingauthor{Faramarz Fekri}{fekri@ece.gatech.edu}

	 
	\icmlkeywords{Machine Learning, ICML}

	\vskip 0.3in
	]
	\myprint{}
	
	\begin{abstract}
		We propose a novel learning paradigm for Deep Neural Networks (DNN) by using Boolean logic algebra.
We first present the basic differentiable operators of a Boolean system such as conjunction, disjunction and exclusive-OR and show how these elementary operators can be combined in a simple and meaningful way to form Neural Logic Networks (NLNs). We examine the effectiveness of the proposed NLN framework in learning Boolean functions and discrete-algorithmic tasks.
We demonstrate that, in contrast to the implicit learning in MLP approach, the proposed neural logic networks can learn the logical functions explicitly that can be verified and interpreted by human. In particular,  we propose a new framework for learning the inductive logic programming (ILP) problems by exploiting the explicit representational power of NLN.  We show the proposed neural ILP solver is capable of feats such as predicate invention and recursion and can outperform the current state of the art neural ILP solvers using a variety of benchmark tasks such as decimal addition and multiplication, and sorting on ordered list.
	\end{abstract}
	\section{Introduction}
	\label{sec:introduction}
	Deep Neural Networks (DNNs) based on Convolution Neural Networks (CNNs) and  Recurrent Neural Networks (RNNs) have improved the state of the art in various areas such as natural language processing \cite{collobert2008unified}, image and video processing \cite{krizhevsky2012imagenet}, and Speech recognition \cite{dahl2012context} just to name a few. 
However, while in theory it is known that DNNs and specifically RNNs can be Turing complete and capable of learning any program \cite{siegelmann1992computational}, there has been limited success in using DNNs for learning algorithmic problems. Even a rather simple decimal multiplication problem is very difficult to learn just by providing the model with input/output pairs of examples \cite{kaiser2015neural}. 
%
%
%
%
In particular, MLP based come with some limitations. These model, in general, do not construct any explicit and symbolic representation of the algorithm they learned and the algorithm is implicitly stored in thousands or even millions of weights, which is typically impossible to be deciphered and verified by human agents. Further, MLP networks are usually suitable for cases where there are many training examples, and usually do not generalize well where there only limited training examples. 
One of the most successful machine learning approaches that addresses these shortcomings for learning discrete algorithmic tasks is the Inductive Logic Programming (ILP). In ILP, explicit rules and symbolic logical representations can be learned using only a few training examples and these models are usually able to generalize well. Further, the explicit symbolic representation that is obtained via ILP can be understood and verified by human, and can also be used to write programs in any conventional programming language. 
Recently there has been some attempts to bridge the gap between the two discipline and to use the deep learning methods in solving the ILP problems. These works usually rely on some forms of transforming the ILP satisfiability problem into a differentiable problem which in turn could be solved by gradient descent algorithms \cite{holldobler1999approximating,bader2008connectionist,francca2014fast,serafini2016logic,evans2018learning}.
In this paper we present an alternative approach to the traditional MLP design for learning Boolean functions and aim to address some of the shortcoming of the MLP for learning discrete-algorithmic tasks. 
Our key idea is to define a set of differentiable Boolean operators that can be combined in a multi-layer cascade design like MLP, and are capable of computing and learning Boolean functions. Unlike MLP, our proposed model provides explicit symbolic representation which could be tested and verified by human.
We further demonstrate that the proposed approach can be used to transform the ILP into a differentiable problem and solve it using gradient optimizers more efficiently than the existing neural ILP solvers. 
The general idea of representing and learning Boolean functions using neural networks is not new. There is significant body of research from the early days of machine learning using neural networks that is focused on the theoretical aspects of this problem. Some of special Boolean functions such as parity-N and XOR has been the subject of special interest as benchmark tasks for theoretical analysis. Minsky and Papert \cite{minsky2017perceptrons,wasserman1989neural} for example showed the impossibility of representing all functional dependence and proved this for XOR logic function while other works demonstrate the possibly of doing so by adding hidden layers. 
From the practical standpoint however, as was suggested by many works (for example \cite{steinbach2002neural}), any Boolean function can be learned by a proper multi layer design equipped with proper activation functions. However, as we will show in following chapters, there are scenarios where they do not perform well. Moreover, even if they learn successfully, it is notoriously difficult to decipher the actual learned Boolean function. 

In contrast, in this paper, we propose a new design for the logical operators (namely Neural Logic Networks (NLN)) by using membership weights without the adjustable bias terms. The NLN network has an explicit representational power which separates the proposed models from the previous works. 
%
In this paper, first, we introduce general purpose conjunction, disjunction and the exclusive OR neurons as the basic elements of the NLN.
We would then demonstrate the properties and characteristics of the proposed model in three areas :
\begin{itemize}[noitemsep,topsep=0pt,parsep=0pt,partopsep=0pt]
\item \emph{Learning Boolean functions efficiently} : In Section \ref{sec:CompareMLP}, we demonstrate how the NLN compares to the MLP in learning Boolean functions.
\item \emph{Generalization} : In Section \ref{sec:LDPC}, 
we compare the generalization performance of NLN and MLP in learning a message passing decoding algorithm for Low Density Parity Check Codes (LDPC) over erasure channels.
\item \emph{Explicit symbolic representation} : 
In Section \ref{sec:ILP}, we propose a new algorithm for solving ILP problems by exploiting the explicit representational power of NLN.
\end{itemize}
%
%
	\section{Neural Logic Layers} 
	\label{sec:LogicLayer}
	\subsection{Neural Conjunction and Disjunction Layers}
	\label{subsec:ConjDisj}
Throughout this paper, we use the extension of the Boolean values to real values in the range $[0,1]$ and we use 1 (True) and 0 (False) representations for the two states of a binary variable. We also define the fuzzy unary and dual Boolean functions of two Boolean variables $x$ and $y$ as:
\begin{subequations}
	\label{eq:BoolAlgebra}
	\begin{align}
	\bar{x} \quad&=\quad  1 - x  \qquad,\quad	x \bigwedge y \quad=\quad   xy  \qquad \\
	x \bigvee y  &= 1 - ( 1 - x )( 1 - y)  
	\end{align}
\end{subequations}
%
This algebraic representation of the Boolean logic allows us to manipulate the logical expressions via Algebra. 
Let $\bx^n \in \{0,1\}^n$ be the input vector in a typical logical neuron.
\begin{figure}[tb]
	\centering
	\subfloat[][]{
		\begin{tabular}{|c|c|c|}
			\hline  
			$x_i$ & $m_i$ & $F_c$ \\ \hline	\hline  
			0 & 0 & 1 \\ \hline
			0 & 1 & 0 \\ \hline
			1 & 0 & 1 \\ \hline
			1 & 1 & 1 \\ \hline
		\end{tabular}
		\label{fig:Fc}%
	}%
	\qquad
	\subfloat[][]{
		\begin{tabular}{|c|c|c|}
			\hline	 
			$x_i$ & $m_i$ & $F_d$ \\ \hline \hline
			0 & 0 & 0 \\ \hline
			0 & 1 & 0 \\ \hline
			1 & 0 & 0 \\ \hline
			1 & 1 & 1 \\ \hline
		\end{tabular}
		\label{fig:Fd}%
	}
	\caption{Truth table of $F_c(\cdot)$ and $F_d(\cdot)$ functions}%
	\label{fig:FcFd}%
\end{figure}
To implement the conjunction function, we would like to select a subset in $\bx^n$ and apply the fuzzy conjunction (i.e. multiplication) to the selected elements. One way to accomplish this is to use a softmax function and select the elements that belong to the conjunction function similar to the concept of pointer networks \cite{vinyals2015pointer}. This requires knowing the number of items in the subset (i.e. the number of terms in the conjunction function) in advance. Moreover, in our experiment we found that the convergence of model using this approach is very slow for larger input vectors. 
Alternatively, we associate a trainable Boolean membership weight $m_i$ to each input elements $x_i$ from vector $\bx^n$. Further, we define a Boolean function $F_c(x_i,m_i)$ with the truth table as in Fig.\ref{fig:Fc} which is able to include (exclude) each element in (out of) the conjunction function. This design ensures the incorporation of each element $x_i$ in the conjunction function only when the corresponding membership weight is $1$. Consequently, the neural conjunction function can be defined as:
\vspace{-2mm}
\begin{align}\label{eq:conj}
O_{conj}(\bx) &= \prod_{i=1}^{n} F_c(x_i,m_i) \,\,\,  \nonumber\\
\text{where, }  F_c(x_i,m_i) &= \overline{x_i \overline{m_i}} = 1 - m_i ( 1 - x_i) \,,
\end{align}
where $O_{conj}$ is the output of conjunction neuron. To ensure the trainable membership weights remain in the range $[0,1]$ we use a sigmoid function, i.e. $m_i = sigmoid( c \,w_i )$ where $c \ge 1$ is a constant. 
Similar to perceptron layers, we can stack $m$ neural conjunction neurons to create a conjunction layer of size $m$. This layer has the same complexity as a typical perceptron layer without incorporating any bias term. More importantly, this way of implementing the conjunction layer makes it possible to interpret the learned Boolean function directly from the values of the membership weights.  

The disjunctive neuron can be defined similarly by introducing membership weights but using the function $F_d$ with truth table as depicted in Fig.\ref{fig:Fd}. This function ensures an output 0 from each element when the membership is zero which correspond to excluding the $x_i$ element from the neuron outcome. Therefore, the neural disjunction function can be expressed as:
\begin{align}\label{eq:disj}
O_{disj}(\bx) = \overline { \prod_{i=1}^{n} \overline {F_d(x_i,m_i)} } &=  1 -  \prod_{i=1}^{n} ( 1 - F_d(x_i,m_i) )  \,, \nonumber \\
\text{where, } F_d(x_i,m_i) &= x_i m_i 
\end{align}
By cascading a conjunction layer with a disjunctive layer, we can create a multi-layer structure which is able to learn and represent Boolean functions using the Disjunctive Normal Form (DNF). Similarly we can construct the Conjunctive Normal Form (CNF) by cascading the two layers in the reverse order.
The total number of possible logical functions over a Boolean input vector $\bx \in \{0,1\}^n$ is very large (i.e. $2^{2^n}$). Further, 
in some cases, a simple clause in one of those standard forms can lead to an exponential number of clauses when expressed in the other form
For example, it is easy to verify that converting  $(x_1 \bigvee x_2 ) \bigwedge (x_3 \bigvee x_4)\bigwedge\dots\bigwedge(x_{n-1} \bigvee x_{n} )$ to DNF leads to $2^{\frac{n}{2}}$ number of clauses. As such, using only one single form of Boolean network for learning all possible Boolean functions is not always the best approach. 
The general purpose design of the proposed conjunction and disjunction layers allows us to define the appropriate Boolean function suitable for the problem. 

\subsection{Convergence and Initialization}\label{subsec:init_weights}
For a single Boolean layer, it can be easily shown that using a small enough learning rate, if we have counter examples in each training batch, they are guaranteed to converge. For example, by examining the conjunction function in (\ref{eq:conj}), it is easy to verify that if $m_i$ is supposed to be 1, we would need a training example with $x_i=0$ and $O_{conj}=1$ to have a negative gradient necessary for adjusting $m_i$ towards $1$. This can be easily verified considering that $\frac{\partial O_{conj}}{\partial m_i} \propto (x_i-1)$. 

The only parameter which we need to adjust for training these layers is the initial values for the membership weights $m_i$ (or corresponding $w_i$). During the experiments, we realized that while the speed of convergence somewhat depends on the initial values for the weights, in moderate size problems, the network is able to find the optimal setting and converges to the desired output. As such, we usually initialize all the weights randomly using normal distribution with zero mean. 
However, in cases where the dimension of the input vector is very large, this type of initialization may result in a very slow convergence in the beginning. Due to the multiplicative design of these layers, when many of the membership variables have values which are not zero or one, the gradient can becomes extremely small. To avoid this situation, we must ensure that most of the membership variables are almost zero in the beginning. In our experiments we usually initialize the membership weights by randomly setting a small subset of inputs to values close to $1$ and we initialize the rest of membership variables to very small constants (e.g. $1e-3$). Alternatively we can initialize weights by a normal distribution with negative mean which needs to be adjusted correctly dependent on the size of the layer.

	\subsection{Neural XOR Layer}
	\label{sec:XOR}
	Exclusive OR (XOR) is another important Boolean function which has been the subject of many researches over the years, especially in the context of parity-N learning problem. It is easy to verify that expressing XOR of an $n$-dimensional input vector in DNF form requires $2^{n-1}$ clauses. Although, it is known that it cannot be implemented using a single perceptron layer \cite{minsky2017perceptrons,duch2006k}, it can be implemented, for example, using multilayer perceptron or multiplicative networks combined with small threshold values and sign function activation \cite{iyoda2003solution}. However, none of these approaches allow for explicit representation of the learned XOR functions directly. Here we propose a new algorithm for learning XOR function (or equivalently the parity-N problem).
%
%
To form the logical XOR neuron, we first define $k$ functions of the form:
\begin{equation}\label{eq:fi}
  \begin{aligned}
f_1 ( \bx) &= x_1  + x_2 + \dots  + x_{k} \,\,\boxed{ -  x_{k+1}   - \dots -  x_{n} }\\
f_2 ( \bx) &= x_1 + x_2 + \dots  \,\,\boxed{- x_{k} -   \dots  - x_{n-1} } + x_{n} \\
&\vdots\qquad\qquad\vdots\qquad\qquad\vdots\qquad\qquad\vdots\qquad\qquad\vdots\qquad\qquad\\
f_k ( \bx) &= x_1 \,\,\boxed{- x_2 - \dots  -x_k  -  x_{k+1} } +\dots  + x_{n}
  \end{aligned}
  \end{equation}
where $k = \frac{n}{2}$ (assuming $n$ is even). Then, we define the XOR function as in Theorem \ref{theorem:xor}.
%
%
%
%
\begin{theorem} \label{theorem:xor}
	Given the set of $k$ functions as defined in (\ref{eq:fi} we have:
\begin{subequations}
\label{eq:compute_xor}
\begin{align}
\text{XOR}(\bx) &= g_1(\bx) \bigwedge g_2(\bx) \bigwedge \dots\bigwedge g_k(\bx)  \label{subeq:xor1} \,,\\  
\text{where,   } g_i ( \bx ) &= 
\begin{cases}
0   \ \ \ \ \mbox{if} \ \ f_i(\bx ) = 0 \\
1   \quad else
\end{cases}\label{subeq:xor2} 
\end{align} 
\end{subequations}
\end{theorem}
\begin{proof}
	See Appendix \ref{proof_th1}.
\end{proof}
Inspired by the Theorem.\ref{theorem:xor}, we design our XOR neuron as:
\begin{align}
O_{XOR}(\bx) &= \prod_{i=1}^{k} hs \bigg(  \, \big\lvert \bx \times  ( \bM_i  \odot  \bw )^T  \big\rvert \bigg)
\end{align}
%
Here, $hs(\cdot)$ is the hard-sigmoid function, and $\times$ and $\odot$ denote matrix and element-wise multiplication correspondingly. Further, vector $M_i \in \{-1,1\}^n$ is the set of coefficients used in $f_i(\bx)$ and $\bw$ is the vector of membership weights.
The resulting XOR logical neuron uses only one weight variable per input element for learning.
However, its complexity is $k$ times higher than the conjunction and disjunction neurons for an input vector of length $n=2k$.

	\section{NLN vs MLP}
	\label{sec:CompareMLP}
%
We now compare the performance NLN vs MLP for the task of learning Boolean functions using two synthetic experiments. 
\subsection{Learning DNF form}
For this experiment, we randomly generate some Boolean functions over a 10 bits input vectors and a randomly generated batches of 50 samples as training data. We train two models; one designed via our proposed DNF network (with 200 disjunction functions) and another designed by two layers MLP network with hidden layer of size 1000 and 'relu' activation and use 'sigmoid' activation function for the output layer. We use ADAM optimizer \cite{KingmaB14} with learning rate of 0.001 for both models and count the number of errors in 1000 randomly generated test samples. 
When we used a Bernoulli distribution with parameter $p=0.5$ (i.e. fair coin toss) for generating the bits of each training samples, both models quickly converge and the number of test error drops to zero for both. However, in many realistic discrete problems, the 0's and 1's are not usually equiprobable. As such, next we use Bernoulli with parameter $p=0.75$. Fig. \ref{fig:dnf_compare} depicts the comparative performance of the two models. The proposed DNF model converges fast and remains at 0 error. On the contrary, the MLP model continues to generate errors. In our experiments, while the number of errors decreases as training continues for an hour, the MLP model never fully converges to the true logical function and occasionally generates some errors. While for some tasks, this may be a negligible error, in some logical applications such as the ILP task (in Section. \ref{sec:ILP}), this behavior prevents the model from learning.
\subsection{Learning XOR function}
Next, we compare the two models for a much more complex task of learning the XOR logic.
We use a multi layer MLP with 'relu' as activation functions in the hidden layers and sigmoid function in the output layer as usual. 
As for NLN, we use a single XOR neuron as described in \ref{sec:XOR}.
For the small size inputs both models quickly converge. However, for larger size input vectors ($n>30$) the MLP model fails to converge at all.
Fig \ref{fig:xor_gate_compare} shows the average bit error over the number of training samples.
%
The error rate for MLP was around $.5$, which indicates it failed to learn the XOR function. On the contrary, the XOR logic layer was able to converge and learn the objective in most of the runs. This is significant considering the fact that the number of parameters in our proposed XOR layer is equal to the input length, i.e., one membership per input variable.
\begin{figure}[htb]
	\centering     
	\subfloat[DNF Task]{\label{fig:dnf_compare}\includegraphics[width=35mm]{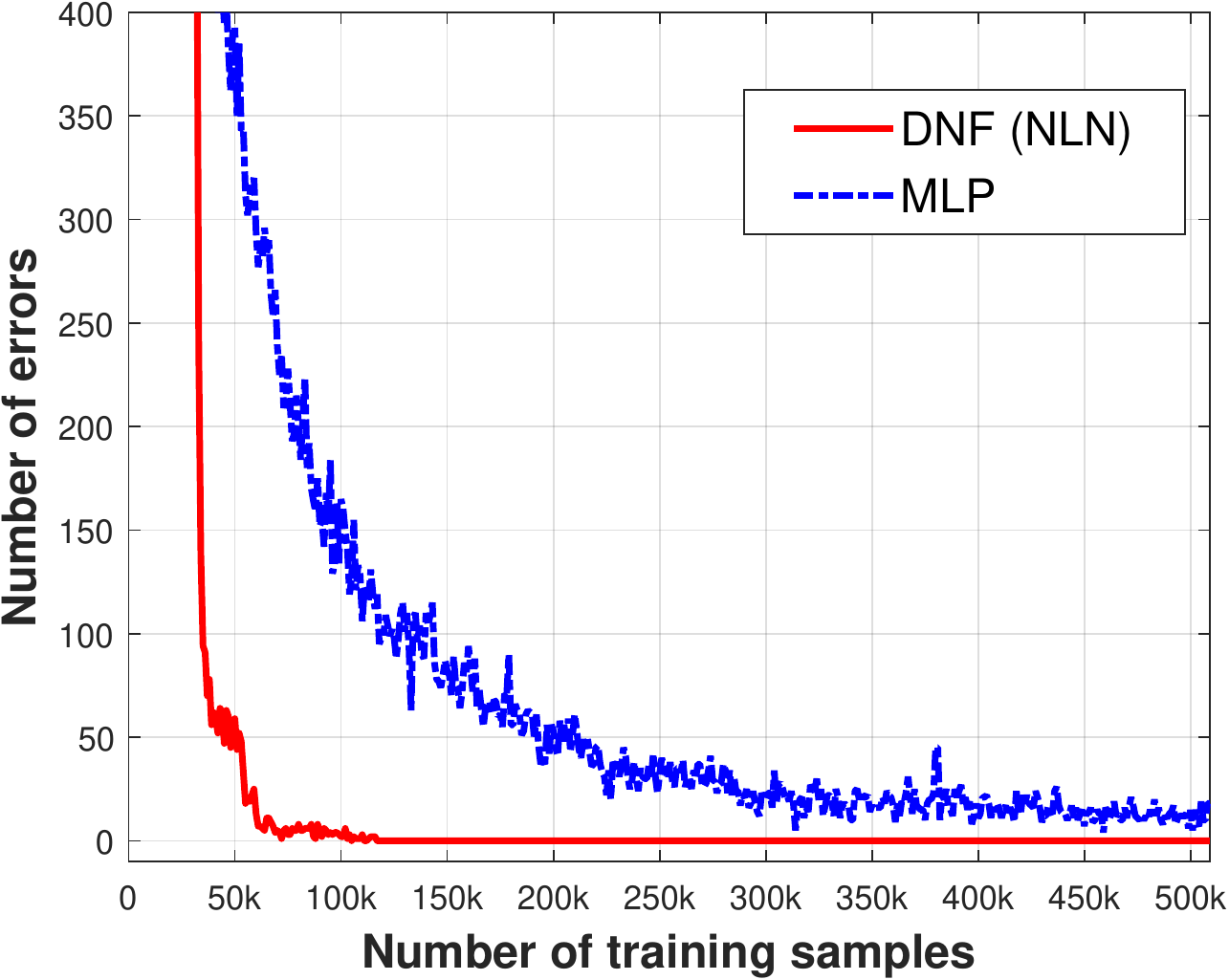}}
	\subfloat[Xor 50 Task]{\label{fig:xor_gate_compare}\includegraphics[width=35mm]{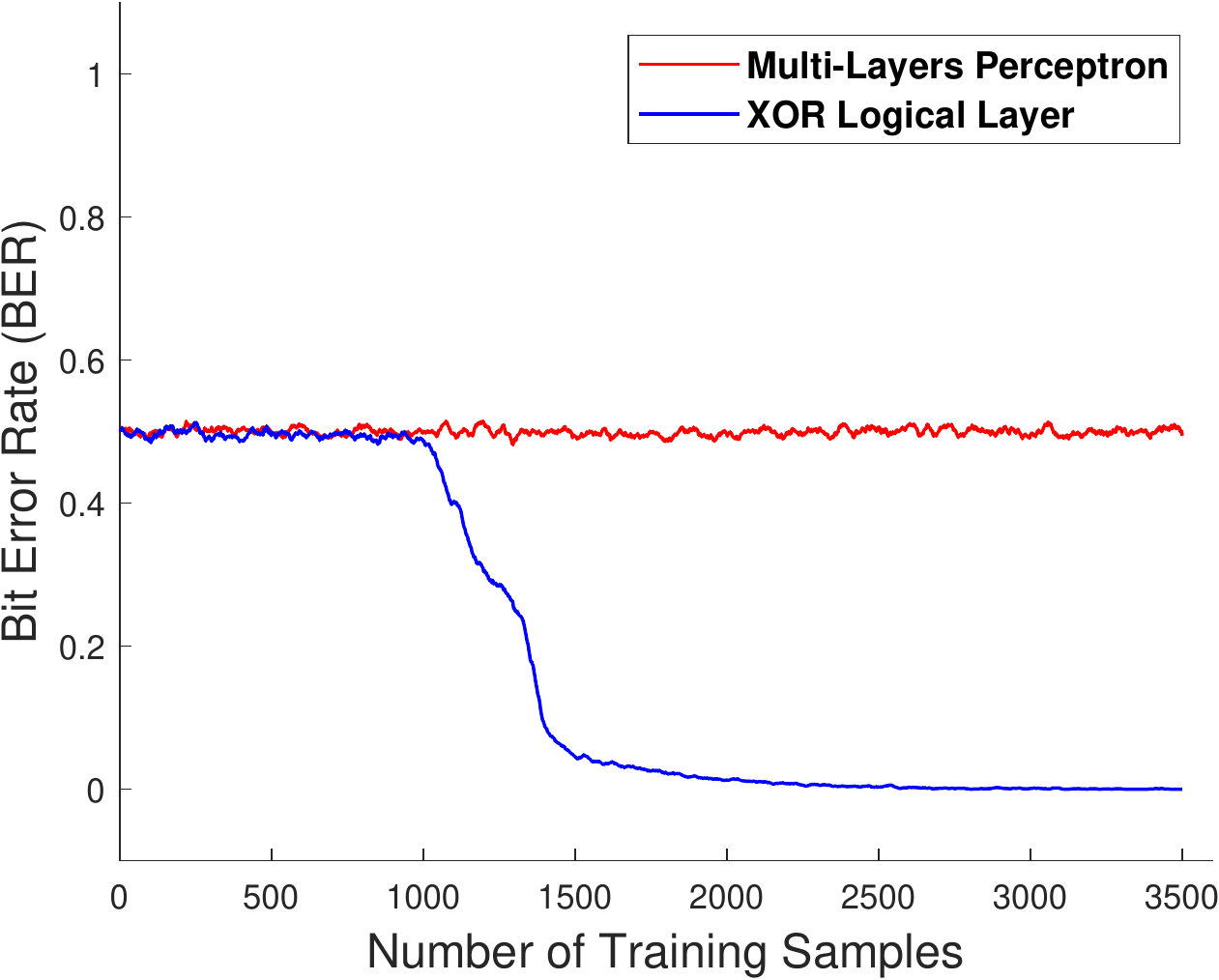}}
	\caption{Comparing MLP vs NLN for learning Boolean functions}
    \label{fig:mlp_compare}
    \vspace{-5mm}
\end{figure}
%
	
	\section{Generalization}
	\label{sec:LDPC}
	To evaluate the generalization, we consider learning an iterative decoding algorithm for the LDPC codes.
LDPC codes are linear error correcting codes that are widely used due to their capacity achieving performance \cite{richardson2001design}. 
%
One popular problem in the coding research is decoding these codes over the Binary Erasure Channel (BEC), where a subset of the bits in the received codeword (from the channel output) is marked as erased due to the channel corruption.
For BEC, decoding of the received LDPC codeword can be performed using an iterative Message Passing (MP) algorithm by enforcing the parity checks in the parity check matrix. 
To compare the performance of MLP vs NLN in learning a discrete-algorithmic task, we use the deep recurrent model that was introduced in \cite{payani2018ISTC} to learn the iterative decoding using MLP and NLN. 

Simply put, in message passing decoding of LDPC codes, each iteration involves a forward and backward path. In the forward path, the content of each check-node is updated via function $\boldsymbol{\mathscr{F}}$ which takes all the connected variable-nodes as input. In the backward path, the content of each variable-nodes is then updated via function $\boldsymbol{\mathscr{B}}$ which takes the signal from all the connected check-nodes as input.

To compare the performance of MLP and NLN we design the forward-backward functions (i.e., $\boldsymbol{\mathscr{F}}$ and $\boldsymbol{\mathscr{B}}$) for the first model using MLP architecture (i.e., LDPC-MLP) and for the second model using NLN (LDPC-NLN).
We use comparable number of parameters in each model (e.g. for LDPC(3,6) code of length 48 we use hidden dimension of size 200). 
In both models, we use randomly generated codewords for a regular LDPC(3,6) code of length 48 as training data and set the number of message passing iteration in training to 3. ($t_{max}=3$). In testing phase, we run the trained models for many more iterations to see how much each model has generalized and learned the iterative algorithm. 

Fig~\ref{fig:LDPC} depicts the performance of two models in terms of bit error probability (BER).
As one may expect, the model based on MLP converges faster and generates lower BER for the setup used in training, i.e, $t=3$. 
However, increasing the number of iterations in test time not only does not improve the accuracy for LDPC-MLP, it even degrades the performance for $t>3$. On the other hand, as the number of iterations increases, the performance of LDPC-NLN model improves significantly.
Arguably, there are ways to improve the performance of MLP in such tasks (e.g, by significantly increasing the number of training iterations and enforcing the network to generate valid outcome at the end of each iterations by adding some penalty term). However, the NLN model provides a more natural way for learning such discrete-algorithmic tasks.
\begin{figure}[tb]
	\centering     
	\includegraphics[width=0.25\textwidth]{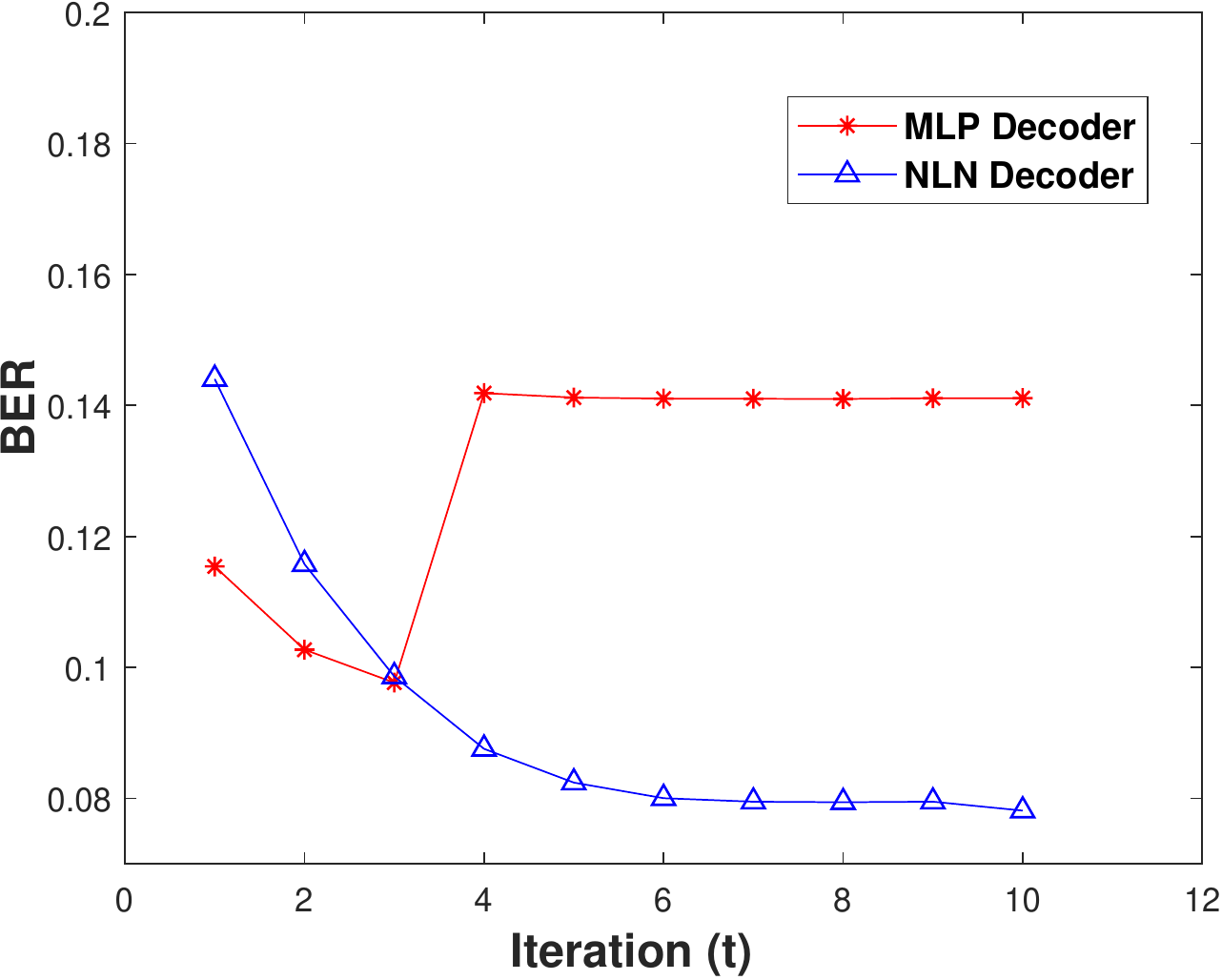}
%
	\caption{LDPC decoding over Erasure Channels}
	\label{fig:LDPC}
	\vspace{-5mm}
\end{figure}

	\section{Induction Logic Programming via NLN}
	\label{sec:ILP}
	One of the recent breakthroughs in solving ILP problems (specially for recusive and algorithmic tasks) is due to works such as \cite{cropper2015logical} which led to the invention of Metagol \cite{metagol}, the state of the art ILP solver capable of learning via predicate invention and recursion. Very recently in (Evans and Grefenstette 2018) the authors proposed a differentiable ILP (dILP) which also supports those features but using a neural network framework. While there are some other noticeable works on neural ILP solvers, we would mainly compare our proposed model to the Metagol and dILP since the other alternatives (for instance \cite{holldobler1999approximating,bader2008connectionist,francca2014fast,serafini2016logic}) do not support both of these important features (i.e., recursion and predecate invention) and therefore are not optimal for solving recursive algorithmic problems.

In this chapter, we introduced a new differentiable ILP solvers by exploiting the explicit representational power of our NLN which we believe is a significant improvement over the dILP and is more flexible than Metagol in terms of the need for an expert input.

For a more complete reference on ILP programming we refer the reader to \cite{muggleton1994inductive,dzeroski2007inductive}. Here, we give a brief background relevant to our proposed algorithm using an example problem. 
Logic programing is a programming paradigm in which we use formal logic (and usually first-order-logic) to describe relations between facts and rules of a program domain. 
In this framework rules are usually written as clauses of the form:
\begin{equation}\label{eq:clause}
H \leftarrow B_1,\,B_2,\,\dots,\,B_m
\end{equation}
where $H$ is called \texttt{head} of the clause and $B_1,\,B_2,\,\dots,\,B_m$ is called \texttt{body} of the clause. A clause of this form expresses that if all the Boolean terms in the \texttt{body} are true, the \texttt{head} is necessarily true.
We assume each of the terms $H$ and $B$ are made of \texttt{atoms}. Each \texttt{atom} is created by applying an $n$-ary Boolean function called \texttt{predicate} to some constants or variables. A \texttt{predicate} states the relation between some variables or constants in the logic program. Throughout this paper we will use small letters for constants and capital letters (A, B, C, ...) for variables. In most ILP systems, each predicate can be defined via several clauses of the form stated in (\ref{eq:clause}) which is equivalent to the DNF logical form. 

Let's consider the logic program that defines the \texttt{lessThan} predicate over natural numbers:
\begin{align} 
lessThan(A,B) &\leftarrow inc(A,B) \nonumber\\
lessThan(A,B) &\leftarrow lessThan(A,C), inc(C,B) \label{eq:lessThan}
\end{align}
and assume that our constants contains the set $C = \{0,1,2,3,4\}$ and the ordering of the natural numbers are defined using the predicate \texttt{inc} (which defines increments of 1). The set of background atoms which describe the known facts about this problem is the set $\mathcal{B} = \{inc(0,1),\, inc(1,2),\, inc(2,3),\, inc(3,4)\}$. 
We associate two scalar functions $arity(p)$ and $var(p)$ corresponding to the number of input arguments for the predicate and the number of variables that can be used in defining predicate. 
Further, we associate a Boolean function $F_p$ to each (intensional) predicate which defines the Boolean function corresponding to the predicate $p$. 
In the above example $arity(lessThan)=2$ and $var(lessThan)=3$ and the predicate function $F_{lessThan}$ can be defined over all possible atoms which involve three variables A,B,C (e.g., it is defined as $F_{lessThan} = inc(A,B) \bigvee (lessThan(A,C) \bigwedge inc(C,B) )$ in (\ref{eq:lessThan})).   

We also distinguish between \text{extensional} and \text{intensional} predicates. The former is entirely defined by the ground facts (eg. $inc$ predicate in the above example), while the latter is defined using the other predicate function (eg. the $lessThan$ predicate in the above example)
Once we have the predicate formula (\ref{eq:lessThan}) which describe our \texttt{target} predicate, we can use rules of deduction and infer all the \texttt{consequences} of the program using  \texttt{forward chain of reasoning}, i.e, we apply the target predicate rules to the constants in the program iteratively. Let $P_i$ be the set of intensional predicates and $X^{(t)}$ be the set of deduced facts at time stamp $t$. We infer the $X^{(T)}$ where $T$ is the number of time stamps using the recursive formula:
\begin{align*}
X^{(i)} = &X^{(i-1)}  \bigcup \\
&\{\,p(a_1,\dots,a_m) |  F_p(a_1,\dots,a_n)=True,\\ 
& a_k \in C , p \in P_i , n=var(p),m=arity(p) \},
\end{align*}
where, $X^{(0)}$ consist of background facts. As an example, for the logic program $lessThan$ we will have:
\begin{align*}
X^{(0)} &= \mathcal{B} = \{inc(0,1),\, inc(1,2),\, inc(2,3),\,inc(3,4)\} \\
X^{(1)} &= X^{(0)} \bigcup \,\{lt(0,1),\,lt(1,2),\,lt(2,3),\,lt(3,4)\}\\
X^{(2)} &= X^{(1)} \bigcup \,\{lt(0,2), \,lt(1,3), \,lt(2,4)\}\\
X^{(3)} &= X^{(2)} \bigcup \,\{lt(0,3), \,lt(1,4)\}\\
X^{(4)} &= X^{(3)} \bigcup \,\{lt(0,4)\},
\end{align*}
where we use $lt$ as shorthand for $lessThan$. Here, applying the predicate rules beyond $t=4$ does not yield any new ground atom. 

Given the background facts ($\mathcal{B})$ and a set of positive and negative examples ($\mathcal{P} \text{ and } \mathcal{N}$ respectively), The goal of ILP is that to learn a program (including target predicate and a number of possible auxiliary predicates) such that it entails all the positive examples and rejects all the negative ones. 
The predicate function defined in (\ref{eq:lessThan}) is one such solution to the ILP problem which satisfies all the examples. 


We use the simple $lessThan$ logic program in above as an example to explain the basics of the proposed algorithm. 
Assume we consider a solution for the predicate function $F_{lt}$ containing at most three variables, i.e. $(A,B,C)$. We define the function $Perm(S,n)$ to return all the tuples of length $n$ from the elements of a set $S$. For example, $Perm(\{A,B\},2$) would give the set $\{(A,A),(A,B),(B,A),(B,B)\}$. Further, for any predicate $p$ and set of variables $V$ we define the set $Terms(p,V)$ as:
\begin{equation}
Terms(p,V) = \{ p(arg) | \,arg \in Perm(V,\,arity(p)\,)\,\}
\end{equation}
For now, if we exclude the use of functions in defining predicates, the set of all the atoms that can be used in defining target predicate can be expressed as:
\begin{align}
InputList (F_{lt}) &= Terms( inc, \{A,B,C\}) \\&\bigcup Terms( lt, \{A,B,C\}) 
\end{align}
This is correspond to the set $\{inc(A,A),\dots,inc(C,C),\, lt(A,A),\dots,\,lt(C,C)\}$. 


Most proposed ILP solvers examine only a very limited subset of possible combinations to find a solution. Metagol \cite{metagol}, the state of the art ILP solver based on meta interpretive learning \cite{cropper2015logical}, employs user-defined meta rules to reduce the set of possible combinations of terms. This requires an expert knowledge with regards to the possible forms of the solution, which is a restrictive approach. Further,this approach may require many trials to find the suitable set of meta rules.

Among the neural ILP solvers, the current state of the art solver proposed by \cite{evans2018learning} limits the number of possible terms to all the combinations containing only two atoms which significantly reduces the space of possible solutions and uses a softmax network to find a set of combination corresponding to the answer from all combinations containing only two atoms. While, in principle, this limitation can be alleviated by introducing more and more auxiliary predicates, this approach is not practical and requires huge amount of memory. The sheer number of possible combinations that these algorithms need to consider makes them inviable candidates for larger scale problems specially when it requires recursion and multiple steps of forward chain of reasoning. Consequently in \cite{evans2018learning} the experiments were limited to the predicates with arity of $2$


Our key idea is to employ our NLN framework to define the predicate functions corresponding to each intensional predicate ( instead of limiting the possible terms using search trees or considering limited combinations similar to previous approaches.) This allows for framing the ILP problem as an end-to-end neural network which can be trained via typical gradient based optimization techniques. Further, this would in general eliminate the restrictions for defining the predicate functions. In particular, in NLN we are not limited to use the DNF form for defining the predicate and we can employ any there Boolean network such as a CNF form and XOR logic to learn the predicate functions.



\subsection{NLN based neural ILP Solver}
We present our algorithm using the current $lessThan$ example. 
First, we define the predicate functions for each intensional predicate (only $lt$ here) using NLN. 
For example, we may use the DNF structure in NLN with a hidden layer of 4 (four disjunction terms) to define the $F_{lt}$. 

Next, we define the valuation vector for each predicate at time stamp $t$ as $Y_p^{(t)}$ which consists of (fuzzy) Boolean values of all the ground atoms involving that predicate. For example the vector $Y_{inc}$ includes the Boolean values for atoms in $\{inc(0,0),inc(0,1),\dots,inc(4,4)\}$. Here we remove the $t$ superscripts since the values of atoms from extensional predicate do not change over time. 
\begin{algorithm}[H]
	\KwResult{$Y_{target}^{{T_{max}}} $}
	\For{$t\in \{1,\dots,T_{max}\}$}{
		\For{$p\in P_i$}{
			\For{$arg \in C^ {var(p)}$}{
				$\theta = \{ arg_0/A,\,arg_1/B,\,arg_2/C,\dots\}$ \\
				$x_i = InputList_p |_\theta$\\
				$arg_p = \{arg_0,\dots,arg_{arity(p)-1}  \}$\\
				$Y_p[arg_p] \leftarrow Y_p[arg_p] \,\bigvee \, F_{p}(x_i)$

			}
		}
	}
	\caption{Outline of the NLN based neural ILP solver\label{alg_step} }
\end{algorithm}
Algorithm \ref{alg_step} shows the outline of the $T_{max}$ steps of forward chain of reasoning in the proposed ILP solver algorithm. Here $\theta$ defines a substitution (replacing variables with constants) and $InputList_p |_\theta$ is a fuzzy Boolean vector formed by gathering the corresponding elements of $InputList_p$ function (after substitution of variables with constants) from the content of valuation vectors $Y_p$'s.
In actual tensorflow implementation \cite{tensorflow2015-whitepaper} we reformulate the problem in matrix form and before the start of the training we calculate the content of valuation vectors belong to all extensional predicates to speed up the training. Also, while the algorithm is described sequentially, we compute all the disjunction operations in the inner-most for-loop in parallel in a batch operation since they do not depend on each other. All the conjunction and disjunction operations in our algorithm is implemented as defined in (\ref{eq:BoolAlgebra}). 
	
We use the cross-entropy loss between $\hat{Y}_{target}$ (the ground truth provided by the positive and negative examples) and $Y_{target}^{(T_{max})}$ which is the output of Algorithm.\ref{alg_step}. 

\subsection{Training}
We train the model using ADAM \cite{KingmaB14} optimizer with learning rate of .001 and we initialize membership weights of the NLN using the approach described in section \ref{subsec:init_weights}. After the training is completed, a zero cross-entropy loss indicates that the model has been able to satisfy all the examples in the positive and negative sets. However, there can be some terms with membership weights of '1' in defining each predicate which are not necessary for the satisfiability of the solution. Since there is no gradient at this point we cannot directly remove those term during gradient descent algorithm unless we include some penalty terms. In practice we use a simpler approach.
In the final stage of algorithm we remove these terms by a simple satisfiability check, .i.e, if by setting any of membership variables with value of '1' to '0' the loss function does not change we remove that term from the outcome. Further, because at the end of convergence, the gradient can become small, to speed up the final stage of convergence, when the loss function is below some threshold, we move the membership variables $m_i$ toward binary values by multiplying the correspond weights $w_i$ to some positive constant larger than 1 (e.g. $1.2$ in our experiments) 
\subsection{Benchmark tasks}
We tested the proposed algorithm on the 20 symbolic tasks described in \cite{evans2018learning} and 
the details of these experiments can be found in Appendices $G.1$ to $G.20$ of that paper. In Table \ref{tbl:BENCHMARK_ILP} we have listed the percentages of runs for each of the tasks that resulted in correct solution for the proposed algorithm and compared it to the baseline methods; dILP and Metagol. Although these are rather simple tasks, as shown in Table.\ref{tbl:BENCHMARK_ILP}, the dlp cannot always find a solution in many of the problems. This can be due to the fact that the algorithm depends on the initial weights and therefore many of the simulations may result in poor performance. Metagol, however is a deterministic approach and it can either find a solution or is unable at all. In general, if not provided with carefully tuned meta rules, which define templates for the axillary and target predicates, Metagol cannot learn many of the tasks involving recursion (e.g. Relatedness and Connectedness tasks in Table \ref{tbl:BENCHMARK_ILP}). In contrast, our proposed model can always find the correct solution for these tasks. 
	\vspace{-3mm}
	\begin{table}[h]
	\caption{NLN solver vs dILP and Metagol in benchmark tasks}

	\label{tbl:BENCHMARK_ILP}
	\begin{tabular} {l c  c c}

		Domain/Task   & dILP & Metagol & NLN \\		
		\hline
		Arithmetic/Predecessor &\textbf{100}  & \textbf{100}&  \textbf{100} \\ 
		Arithmetic/Even        & \textbf{100}  &\textbf{100}  &   \textbf{100}\\   
		Arithmetic/Even-Odd    & 49  & \textbf{100} & \textbf{100}  \\  
		Arithmetic/Less than   &\textbf{100}   & \textbf{100} & \textbf{100}  \\   
		Arithmetic/Fizz        & 10  & \textbf{100} &\textbf{100}  \\   
		Arithmetic/Buzz        & 35  &\textbf{100}  & \textbf{100}  \\    
		List/Member      & \textbf{100}  & \textbf{100} & \textbf{100}  \\   
		List/Length      & 93  & \textbf{100} & \textbf{100}  \\   
		Family Tree/Son         &\textbf{100}   &\textbf{100}  &  \textbf{100} \\ 
		Family Tree/GrandParent & 97  &\textbf{100}  & \textbf{100} \\   
		Family Tree/Husband     & \textbf{100}  & \textbf{100} &  \textbf{100} \\   
		Family Tree/Uncle       & 70  & \textbf{100} &  \textbf{100} \\  
		Family Tree/Relatedness & \textbf{100}  &0  & \textbf{100}  \\   
		Family Tree/Father      & \textbf{100}  & \textbf{100} &  \textbf{100} \\    
		Graph/Undirected Edge & \textbf{100}  & \textbf{100} & \textbf{100}  \\   
		Graph/Adjacent to Red &  51 & \textbf{100} &  \textbf{100} \\   
		Graph/Two Children   &  95 & \textbf{100} & \textbf{100}  \\   
		Graph/Graph Colouring&  95 & \textbf{100} & \textbf{100} \\   
		Graph/Connectedness&  \textbf{100} & 0 & \textbf{100}  \\   
		Graph/Cyclic&  \textbf{100} & 0 &  \textbf{100} \\   
	\end{tabular}
\end{table}
\vspace{-4mm}
\subsection{Learning Decimal Arithmetic}
The 20 tasks in the previous section are rather simple tasks.
For more complex tasks for which the predicate definition requires more atoms and the arity of predicates are higher than two, methods such as dILP cannot be used. Even Metagol can only learn such tasks that require recursion when the appropriate rule templates for the are provided by an expert. Here, we apply our method to learn more complex recursive arithmetic tasks. We first describe the addition problem for the natural number domain and then we use the addition predicate as background knowledge in second task to learn the multiplication. 
\subsubsection{Addition Task}
We use $C=\{0,1,2,3,4,5\}$ as constants and our background knowledges is consist of $\mathcal{B}=\{zero(0), eq(0,0), \dots,eq(4,4), inc(0,1), \dots,inc(3,4)\}$, where  $inc$ defines increment and $eq$ tests for equality. The target predicate is $add(A,B,C)$ and we allow for use of two additional variables (i.e., $var(add)=3+2=5$). As usual, we use a DNF network for learning $F_{add}.$ 
One of the solutions that our model finds is:
\begin{align*}
add(A,B,C) &\leftarrow zero(B), \, eq(A,C) \\
add(A,B,C) &\leftarrow add(A,D,E), \, inc(D,B), \,inc(C,E)
\end{align*}
\subsubsection{Multiplication Task}
Next, we add the learned addition predicate to the background knowledge of the previous experiment and then try to learn the $mul(A,B,C)$ predicate. One of the obtained solutions is:
\begin{align*}
\vspace{-3mm}
mul(A,B,C) &\leftarrow zero(B), \, zero(C) \\
mul(A,B,C) &\leftarrow mul(B,A,C) \\
mul(A,B,C) &\leftarrow mul(A,D,E),  inc(D,B), plus(E,A,C)
\end{align*}
It is worth nothing that  to the best of our knowledge, learning recursive algorithmic tasks like this using only positive and negative examples and without using any template for defining viable option (other than assuming DNF form) is beyond the power of any current ILP solver. Indeed, most neural ILP solvers are either incapable of learning recursion or like dlp have limited scope and cannot be used to learn complex predicates. While, tasks such as decimal and binary addition and multiplications can be learned with very sophisticated neural algorithms such as \cite{kaiser2015neural}, they lack the generalization power of ILP and their performance significantly drops when the size of the problem grows. Furthermore, the learned algorithm is not explicit in nature and acquired knowledge cannot be transfered to another problem easily.
\subsection{Sorting an ordered list}
The sorting tasks is more complex than the previous tasks since it requires the list semantics. We implement the list semantic by allowing the use of functions in defining predicates. For a data of type \texttt{list}, we define two functions $H(X)$ and $t(X)$ which allow for decomposing a list into head and tail elements, i.e $A=[H(A)|t(A)]$.
We use elements of $\{a,b,c,d\}$ and all the lists made from permutations of up to three elements as constants in the program. We use extensional predicates such as $gt$ (greater than), $eq$ (equals) and $lte$ (less than or equal) to define ordering between the elements as part of the background knowledge. We allow for using two additional variables in defining the predicate $sort(A,B)$. One of the solution that our model finds is:

\topskip=1pt
\begin{align*}
\noindent 
sort(A,B) &\leftarrow sort(H(A),C),\, lte(t(C),t_A), \\ &eq(H(B),C),\,eq(t_A,t_B) \\
sort(A,B) &\leftarrow sort(H(A),C),\, gt(t(C),t_A),\, eq(t(B),t(C)),\, \\ & eq(H(D),H(C),\, eq(t(A),t(D),\, sort(D,H(B))
\end{align*}
To the best of our knowledge, learning a recursive solution like this which involves clauses with 6 atoms and include 4 variables (and their functions) is beyond the power of any existing neural ILP solver. 
	\section{Conclusion}
	\label{sec:conclusion}
	We have introduced NLN as a new paradigm of neural networks designed for explicit learning and representation of Boolean functions. Using various experiments we showed their effectiveness in learning the logical representations. Further, we demonstrated their generalization superiority to the traditional MLP in a discrete iterative algorithmic. Finally, by proposing a new algorithm for learning ILP problems we demonstrated the importance of the explicit logical representation that is achieved using NLN.

	\appendix
	\section{ Proof of Theorem 1.}
	\label{proof_th1}
	\begin{proof}
First consider that for the case where
the number of $'1'$s in $\bx$ is odd (i.e. $XOR(x)=1$), none of the functions $f_i(\bx)$ can be equal to zero since the sum of odd number of elements from the set $\{-1,1\}$ cannot be zero. Therefore, the statement in (\ref{subeq:xor1}) is true due to (\ref{subeq:xor2}).
Now consider the case that $XOR(\bx)$ is zero. We must show that at least one of the $k$ functions $f_i(\bx)$ would be equal to zero in this case. 
%
Let $M_i \in \{-1,1\}^n$ be the vector of coefficients for $f_i(\bx)$ 
and $s$ be the number of ones in the input vector $\bx$. Further, for any  $f_i$, let $n_i^{(1)}$ and $n_i^{(-1)}$ be the number of corresponding 1 and -1 coefficients that matches the positions of elements of '1' in vector $\bx$. 
We notice that the sign of exactly two elements in $M_i$ and $M_{i+1}$ changes when we go from  $f_i(\bx)$ to $f_{i+1}(\bx)$ and those signs remain unchanged in the next set of functions. As we have $k$ functions and $s \le 2k$, this would guarantee that the sign of the coefficients corresponding to '1' elements changes exactly once in the set of $k$ functions. Thus, in one of the functions, let's say the one corresponding to the $j^{th}$ coefficient vector we would have $n_j^{(1)}=n_1^{(-1)}$ and $n_j^{(-1)}=n_1^{(1)}$ which means  $f_j(\bx)=- f_1(\bx)$. Since the difference between each consecutive $f_i$ can be zero or $\pm2$, this guarantees that at some point one of the $f_i$'s ($1 \le i \le j$) should be equal to zero.

In the above arguments we assumed $n$ is an even number. 
However, if $n$ is an odd number, we can modify it to the $n+1$ problem by appending an extra ’$0$’ entry to the input vector $\bx$. Since ’$0$’ has no effect on the results of XOR, the above arguments still hold.
\end{proof}

	\bibliographystyle{icml2019}
	\bibliography{refs_payani}
	
	\
\end{document}